\documentclass[letterpaper, 10 pt, conference]{ieeeconf}
\IEEEoverridecommandlockouts
\usepackage[utf8]{inputenc}
\usepackage[english]{babel}
\usepackage{subfig}
\usepackage{amsmath}
\usepackage{amsfonts}
\usepackage{amssymb}
\usepackage{mathtools}
\usepackage{graphicx}
\usepackage{xcolor}
\usepackage[capitalise,nameinlink,noabbrev]{cleveref}

\graphicspath{{./figures/}}
\newtheorem{theorem}{Theorem}[section]

\newtheorem{lemma}[theorem]{Lemma}
\newtheorem{assumption}{Assumption}

\newtheorem{definition}{Definition}

\newcommand{\Bez}{B\'ezier }

\title{\LARGE \bf
 Intent-Aware Probabilistic Trajectory Estimation for Collision Prediction with Uncertainty Quantification
}

\author{Andrew Patterson, Arun Lakshmanan and Naira Hovakimyan
\thanks{Andrew Patterson, Arun Lakshmanan, and Naira Hovakimyan
        are with the Department of Mechanical Science and Engineering,
        University of Illinois at Urbana-Champaign, Urbana, IL 61801, USA
        {\tt\small \{appatte2, lakshma2, nhovakim\}@illinois.edu}}%
}

\begin{document}
\maketitle
\thispagestyle{empty}
\pagestyle{empty}
\begin{abstract}
Collision prediction in a dynamic and unknown environment relies on knowledge of how the environment is changing. Many collision prediction methods rely on deterministic knowledge of how obstacles are moving in the environment. However, complete deterministic knowledge of the obstacles' motion is often unavailable.
This work proposes a Gaussian process based prediction method that replaces the assumption of deterministic knowledge of each obstacle's future behavior with probabilistic knowledge, to allow a larger class of obstacles to be considered. 
The method solely relies on position and velocity measurements to predict collisions with dynamic obstacles. We show that the uncertainty region for obstacle positions can be expressed in terms of a combination of polynomials generated with Gaussian process regression. 
To control the growth of uncertainty over arbitrary time horizons, a probabilistic obstacle intention is assumed as a distribution over obstacle positions and velocities, which can be naturally included in the Gaussian process framework.
Our approach is demonstrated in two case studies in which (i), an obstacle overtakes the agent and (ii), an obstacle crosses the agent's path perpendicularly. In these simulations we show that the collision can be predicted despite having limited knowledge of the obstacle's behavior.
\end{abstract}

\section{INTRODUCTION}
\label{s:intro}
When deploying autonomous systems that are safe, it is important to not only ensure the safety of the vehicle in a static environment, but also to define interactions with other agents.
These agents can be cooperative or non-cooperative depending on each agent's ability or desire to communicate and compromise. 
When agents are cooperative, we can take advantage of inter-agent communication to ensure that performance objectives are achieved. Methods that take advantage of agent cooperation are described by the authors of~\cite{kaminer17} and~\cite{bilal19}.
Even without cooperation, we expect an autonomous system to be able to execute a planned trajectory faithfully while avoiding collisions with other agents. 
A high level planner can re-plan around predicted collisions by adjusting the planned trajectory of the vehicle such that no collision is predicted. This planning is expected to occur after mission start and before a collision is imminent. Once collision is imminent, feedback-based methods can be used to guarantee collision avoidance, as found in~\cite{marinho18} and~\cite{cichella17}.

Re-planning is necessary in a changing environment; 
however, to effectively re-plan, an expected future state of the environment is necessary.
Without knowledge of these trajectories, an estimate must be constructed.
Given deterministic knowledge of a vehicle's trajectory, collision can be avoided using the methods described in~\cite{bilal17},~\cite{paden16} and~\cite{wurts18}.
In this work, we consider a single obstacle and a single agent attempting to avoid the obstacle. Predicting the future behavior of an obstacle is difficult with partial or imperfect knowledge of the obstacle's dynamics and intentions.
With knowledge of the obstacle dynamics and control inputs, we can predict the obstacle's future behavior accurately. 
If the control inputs are not known but the dynamics are well understood, then the future behavior can be quantified in terms of reachability sets as proposed in~\cite{zhou16}. Even with knowledge of the obstacle dynamics, as the the prediction time horizon increases, the reachability set of many systems will grow to cover the entire space. Methods designed to predict collision using only this information tend to be very conservative, suitable for trajectory estimation on a short time scale. When even less information is known about the obstacle dynamics, the problem becomes  increasingly difficult to solve. 

To estimate the obstacle's future behavior, without being overly conservative, the estimation method can include some knowledge of the agent's intention.
Methods of predicting the agent's intention with Markovian models are given in~\cite{bandy13}, which uses a mixed observable Markov decision process, and in~\cite{vasquez09}, which uses growing hidden Markov models. The authors of~\cite{bandy13} consider an urban autonomous driving context, where pedestrian intentions need to be identified in real time based on past observations. The intentions are a finite set of points in the environment, and the planner assigns probabilities to each point. The authors of~\cite{vasquez09} extend the Markovian model in a way that allows both the probability of each intention and the set of intentions to be updated at every time-step.

Once the intention is established, probabilistic trajectory estimates can be generated in different ways, varying in complexity and dependence on the knowledge of obstacle dynamics.
Some of the simplest are constant extrapolation, such as the constant heading and velocity assumptions as in~\cite{miller08}. These methods are fast, easy to implement and perform well in situations where the obstacle is assumed to be in steady-state. The intention of these obstacles is implicitly modeled as a desire to continue previous behavior.
Other methods use the knowledge of obstacle dynamics to propagate the current state into the future, rather than constant extrapolation. An example of this type of method can be found in~\cite{yokoyama18}, in which state propagation requires knowledge of agent dynamics. 
The author of~\cite{yokoyama18} estimates the obstacle's intention using knowledge of its optimization method and cost function. With the assumption that the obstacle is attempting to minimize a certain cost function, the inverse of the optimization procedure is computed to predict the optimal trajectory.
This trajectory is then used as the intention.
In addition to model-based methods, there are data based methods that predict obstacle trajectories by propagating the obstacle state through Monte-Carlo methods, where the state propagation is learned from a large data-set of previous behaviors. These methods typically incur a large, upfront training cost which can be prohibitively expensive without a representative data-set. In~\cite{hamlet15}, a pre-trained Bayesian dynamic network is assumed, and the prediction is performed by considering 100 simulations of the Bayesian network and around 20 look-ahead time steps.
The method presented in~\cite{ellis09} uses a large data-set of known pedestrian trajectories to create a velocity field map of the area in question. Trajectory estimation is then performed by using 250 particles. In these methods, the intention is encoded naturally by the data. Clustering the final locations of the particles would return intentions in the form of a distribution over particles. 
Each of these methods can be used for collision prediction, but the collision prediction step becomes more expensive or has limited accuracy, depending on the number of points checked.

To avoid the trade-off between accuracy and sampling seen in these methods, the authors of~\cite{bilal16} and~\cite{usenko17} use \Bez and B-spline basis functions for fast collision prediction.
Using these basis functions, collision prediction can be quickly and analytically checked and a new plan can be computed by modifying the trajectory parameters.
The method used in this paper for collision prediction is detailed in~\cite{arun19} and can be implemented with any absolutely continuous curve.

To take advantage of this computationally efficient method, we use a data-based method that produces predictions of a known basis.
In this paper, we present a Gaussian process regression method for predicting obstacle trajectories given a probabilistic intention estimate. By choosing an appropriate covariance function, we can fix the basis of the mean and variance functions that allows collisions to be quickly predicted. This paper has three main contributions:
\begin{enumerate}
    \item a data-based method for extrapolating obstacle trajectories,
    \item a method for incorporating a probabilistic intention into the estimation method, and
    \item we show that for the cubic-spline kernel, we obtain an uncertainty region as a combination of polynomials, and the square-root of polynomials  can be used for collision prediction.
\end{enumerate}

In Section~\ref{s:prelim}, we present preliminary information on Gaussian processes, multi-output Gaussian processes and collision prediction. In Section~\ref{s:prob}, the necessary definitions and assumptions are provided and we present the problem statement. In Section~\ref{s:method}, we present the methods and analysis results for trajectory estimation with a probabilistic intent. Finally, in Section~\ref{s:results} we demonstrate the method in two collision avoidance scenarios.

\section{PRELIMINARIES}
\label{s:prelim}
\subsection{Gaussian Process}
Gaussian process regression is a data based regression method  that  quantifies the uncertainty of predictions and allows the basis of the trajectory estimate to be chosen. An overview of Gaussian process regression can be found in~\cite{rasmussen06}, which provides general information on the regression method, and~\cite{roberts12}, which focuses on timeseries modelling. In~\cite{roberts12}, the authors discuss the model design choices that incorporate domain knowledge, such as the choice of covariance functions, updating methods and hyperparameter optimization.

A Gaussian process is a collection of random variables, where any finite numbers are jointly Gaussian. We can define these processes as a distribution over functions on $\mathbb{R}^d$ as
\begin{align*}
    F \sim \mathcal{GP}(M,K),
\end{align*}
where $M : \mathbb{R}^d \rightarrow \mathbb{R}$ is the mean function, and $K:\mathbb{R}^d \times \mathbb{R}^d \rightarrow \mathbb{R}$ is a symmetric positive definite covariance kernel function.
The Gaussian process condition is then satisfied if for any finite set of sample times $T=\{ t_1,\dots t_n\}$, the function evaluated at the times in $T$ are samples from a multi-variate Gaussian distribution, that is,
\begin{align*}
    F(T) \sim \mathcal{N}(M(T),K(T,T)).
\end{align*}
For $d=1$, the mean and covariance functions, evaluated at $T$, are a vector and a matrix respectively.
The properties of the Gaussian process are completely determined by the mean function and covariance function. The mean function is often taken to be identically zero; in this case the predictive distribution for a test time, $t$, is given by 
\begin{align}
    p(F(t)\, | \, D,t) &= \mathcal{N}(\mu(t),\sigma^2(t)),
        \label{eq:gp_prob}
    \\
    \mu(t) &= K(T,t) ^\top ( K(T,T) + \Sigma^2)^{-1} Y\nonumber
    \\
    \sigma^2(t) & = K(t,t) \nonumber
    \\
    &- K(T,t) ^\top( K(T,T) + \Sigma^2)^{-1}K(T,t),\nonumber
\end{align}
where $D$ is a collection of times and corresponding outputs $D=\{T,\,Y\}$. 
The individual measurements, elements of  $Y$, are given for every sample time in $T$, so we have $Y=[y_1, \dots,y_n]^\top$. Each element is normally distributed.
The measurement covariance matrix is $\Sigma^2$.  This matrix is a diagonal for independent measurements.
The matrix $P = ( K(T,T) + \Sigma^2)^{-1}$ is called the precision matrix.

The covariance function used in this paper is the cubic spline covariance function, 
\begin{align}
    &k_f(t,t^{\top}) \nonumber
    \\
    &= \theta^2_f \left[\frac{1}{3} \textnormal{ min}^3(\tilde{t},\tilde{t}^{\top}) +\frac{1}{2}\left|t-t^{\top}\right|\textnormal{ min}^2(\tilde{t},\tilde{t}^{\top})\right],
    \label{eq:cubic}
\end{align}
where $t$ is any input time, $\tilde{t}\coloneqq t+\tau$, is the time shifted input, shifted by a constant $\tau>0$, such that the covariance function is positive semi-definite (in this paper $\tau=11$). 
The scaling hyperparameter is denoted $\theta_f$.
The use of the cubic-spline covariance function is a natural choice in the domain of dynamically defined trajectories since they correspond to the solution of double integrator systems with piece-wise constant input commands. Even if the system is not a double integrator, a double integrator is a common simplified model used in systems with only a force input. 
 This kernel is discussed in depth in~\cite{mahsereci18},~\cite{wahba90} and~\cite{wendland05}.

\subsection{Multi-Output Gaussian Process}
Generally, a Gaussian process may be defined on any $d$-dimensional space. In this paper, we will consider multi-output Gaussian processes. The first output is the predicted mean given a data-set, and the second is time-derivative of the data-set. Since differentiation is a linear operator and our covariance function is differentiable, we can perform regression with a single Gaussian process:
\begin{align}
F=
    \begin{bmatrix}
    f \\ f'
    \end{bmatrix}
    \sim 
    \mathcal{GP}\left(
    \begin{bmatrix}
    \mu_{f} \\ \mu_{f'}
    \end{bmatrix},
        \begin{bmatrix}
    k_f & k_f^\partial \\ ^\partial k_f & ^\partial k_f^\partial
    \end{bmatrix}
    \right), 
    \label{eq:mimo}
\end{align}
where the derivative of a covariance function is given by
\begin{align*}
    ^{\partial(i)} k_f ^{\partial(j)} = \frac{\partial^{i+j}k_f(t,t^{\top})}{\partial t^i \partial t^{\top j}}.
\end{align*}
In this multi-output model, the mean function is a vector of two functions, and the covariance is a matrix of functions.  Note that the subscripts $f$ and $f'$ can be replaced with either $x$ or $y$ to indicate spatial dimensions. 
For the cubic spline covariance function in Equation~\eqref{eq:cubic}, these derivatives are given by the authors of~\cite{mahsereci18} as
\begin{align*}
    &k_f ^\partial(t,t^{\top}) = 
    \\&\quad\theta_{f}\theta_{f'}\left[\mathbb{I}(t<t^{\top})t^2/2 + \mathbb{I}(t\geq t^{\top})(tt^{\top}-t^{\top 2}/2)\right],
    \\&^\partial k_f (t,t^{\top}) = 
    \\&\quad\theta_f\theta_{f'}\left[\mathbb{I}(t^{\top}<t)t^{\top 2}/2 + \mathbb{I}(t^{\top}\geq t)(tt^{\top}-t^2/2)\right],
    \\
    &^\partial k_f ^\partial(t,t^{\top}) = \theta_{f'}^2\textnormal{min}(t,t^{\top}),
\end{align*}
where $\mathbb{I}$ is the indicator function and $\theta_{f'}$ is the function derivative scaling hyperparameter. 
Additional information on this joint estimation scheme is presented in~\cite{rasmussen06},~\cite{roberts12} and~\cite{eriksson18}.
\subsection{Minimum Distance}
Minimum distance calculation is central to collision prediction. To avoid collision, we wish to keep the minimum distance between the agent trajectory and a specified obstacle uncertainty region greater than a safety distance. 
In general, 
given parametric equations $\beta: \mathcal{T} \to \mathbb{R}$ and $\delta: \mathcal{T} \to \mathbb{R}$ defined over a closed interval $\mathcal{T} \subset \mathbb{R}$, we define the $\delta$-region around the parametric equation $\beta$ as
\[
\mathcal{B}^\delta(t) = \{x \in \mathbb{R} \ : \  \beta(t) - x \le | \delta(t) | \}.
\]
The minimum distance between a parametric equation $\alpha: \mathcal{T} \to \mathbb{R}$ and $\mathcal{B}_\delta$ is given by 
\begin{equation*}
    d_\textrm{min}(t, \alpha, \mathcal{B}^\delta) = \min_{b \in \mathcal{B}^\delta(t)} | \alpha(t) - b |,
\end{equation*}
for all $t \in \mathcal{T}$.
\section{PROBLEM FORMULATION}
\label{s:prob}
Consider a mission in which an agent must navigate a planar environment in the presence of a dynamic obstacle. 
The obstacle is non-cooperative and moves through the $x$-$y$~plane without providing information about its future trajectory to the agent. 
Based on the obstacle's motion, the agent must predict if its planned trajectory will cause a collision. We formalize the scenario with the following definitions.
\subsection{Definitions and Assumptions}
\begin{definition}[Agent]
 The agent, $A$, is  represented by a coordinate in the $x$-$y$ plane and its safety distance $\Delta_\textrm{safe}>0$.
\end{definition}
\begin{assumption}[Agent Trajectory]
    We assume that the agent has a known trajectory, $\psi:\mathcal{T}\rightarrow \mathbb{R}^2$, for any time in $\mathcal{T}\subset \mathbb{R}$.
    The $x$ component of this trajectory is denoted by $\psi_x$, and the $y$ component is referred to as $\psi_y$.
\end{assumption}

\begin{definition}[Obstacle]
    The obstacle, $O$, is represented by a coordinate in the $x$-$y$ plane.
\end{definition}
\begin{assumption}[Initial Separation]
    Assume that the agent trajectory is initially separated from the obstacle location. This separation must be larger than $\Delta_\textrm{safe}$.
\end{assumption}
\begin{definition}[Intention]
We define the probabilistic intention, $I$, as a distribution over positions and velocities at a future time,  the intention time, $t_I$. Note that $x$ and $y$ indicate positions, and $x'$ and $y'$ indicate velocities.
\end{definition}
\begin{assumption}[Intention Distribution]
    For each dimension in the plane, indicated by subscript $x$ or $y$, we assume that the intention follows a normal distribution:
    \begin{align*}
       & I_x \sim \mathcal{N}(\mu_{Ix},\sigma^2_{Ix}), \qquad I_y \sim \mathcal{N}(\mu_{Iy},\sigma^2_{Iy})
        \\
        &I_{x'} \sim \mathcal{N}(\mu_{Ix'},\sigma^2_{Ix'}), \qquad I_{y'} \sim \mathcal{N}(\mu_{Iy'},\sigma^2_{Iy'}),
    \end{align*}
    and that these values are known.
\end{assumption}
\begin{assumption}[Constant Intention]
    For the duration of the collision avoidance task, it is assumed that the intention is constant, i.e. the means, variances and intention time are fixed.
\end{assumption}
\begin{assumption}[Dynamic Behavior]
    We assume that the motion of the obstacle is governed by a differential equation.
\end{assumption}
\begin{definition}[Expected Position]
    In~\cref{eq:mimo}, we see that the mean function is a vector of two functions. 
    The expected position of the obstacle is, $\mu_f$, the first element in this vector of functions. Note that $f$ can be replaced with either $x$ or $y$ to indicate spatial dimension. 
\end{definition}
\begin{definition}[Position Variance]
    In~\cref{eq:mimo}, we see that the covariance function is a matrix of four functions. 
    Applying the update equations in~\cref{eq:gp_prob} to this matrix of functions yields $\sigma^2(t)$, another matrix of functions:
    \begin{align*}
        \sigma^2(t) = 
        \begin{bmatrix}
           \sigma^2_f(t) &  \sigma^2_{f,f'}(t)
           \\
             \sigma^2_{f',f}(t) &\sigma^2_{f'}(t) 
        \end{bmatrix}.
    \end{align*}
    The position variance of the obstacle is, $\sigma^2_f$, the first element in this matrix of functions. Note that $f$ can be replaced with either $x$ or $y$ to indicate spatial dimension. 
\end{definition}
\begin{definition}[Uncertainty Region]
\label{def:ur}
    We define the uncertainty region for each dimension as
    \begin{align*}
        &\Psi_{x}^{2\sigma}(t)=\{x\in\mathbb{R} \ : \ \mu_x(t) - x  \leq |2\sigma_x(t)| \} \quad \textnormal{and}
        \\
        &\Psi_{y}^{2\sigma}(t)=\{y\in\mathbb{R} \ : \ \mu_y(t) - y  \leq |2\sigma_y(t)| \},
    \end{align*}
    where $\mu(t)$ and $\sigma(t)$ are the posterior mean and standard deviation for each dimension, as defined in Equation~\eqref{eq:gp_prob}. 
\end{definition}
\begin{assumption}[Sequential Data]
    The data are collected sequentially, that is the elements of the time vector, $T=[t_1,\dots,t_n]$ and the corresponding measurements are ordered such that $t_k<t_{k+1}$ for a positive integer $k<n$.
    \label{a:sequential}
\end{assumption}
\subsection{Time Intervals}
While the trajectory of the obstacle is unknown for future times, we assume that the trajectory can be measured at discrete time instances between time $t_a$ and $t_b$, called the observation time interval, $\mathcal{T}_O = [t_a,t_b]$. During this interval, it is assumed that we have a set of noisy measurements of the vehicle's position and velocity.
The prediction time interval is the time interval between the last measurement and the intention time: $\mathcal{T}_P =(t_b,t_I]$. The union of these intervals is $\mathcal{T}$, the time interval of interest.
\subsection{Problem Statement}
Given full knowledge of the agent's planned trajectory on the time interval of interest $\mathcal{T}$, a data-set ${D}$ of measured positions and velocities of the obstacle on the observation time interval $\mathcal{T}_O$, and a probabilistic intention $\mathcal{I}$, we  wish to determine whether or not the agent will collide with the uncertainty region, that is: 
\begin{align}
    \label{eq:prob}
    \sqrt{d_\textrm{min}^2(t, \psi_x, \Psi_x^{2\sigma}) + d_\textrm{min}^2(t, \psi_y, \Psi_y^{2\sigma})} > \Delta_\textrm{safe},
\end{align}
for any $t\in\mathcal{T}_O$. If this inequality is violated, we say collision has occurred.
\section{METHOD}
\label{s:method}
\subsection{Covariance Choice}
Recall that our Gaussian process is defined in terms of a mean and joint position-velocity covariance function:
\begin{align}
F=
    \begin{bmatrix}
    f \\ f'
    \end{bmatrix}
    \sim 
    \mathcal{GP}(
    \begin{bmatrix}
    \mu_{f} \\ \mu_{f'}
    \end{bmatrix},
    \begin{bmatrix}
    k_f & k_f^\partial \\ ^\partial k_f & ^\partial k_f^\partial
    \end{bmatrix}
    ).
    \label{eq:covF}
\end{align}
This Gaussian process can be updated through Gaussian process regression using Equation~\eqref{eq:gp_prob}. While the position and velocity information of the vehicle are assumed to be correlated by the choice of covariance function, the spatial coordinates are assumed to be independent. This assumption allows the trajectory to be estimated in each dimension separately and does not require additional assumptions on dynamic correspondences between the dimensions.
These independent Gaussian processes are given by the equations:
\begin{align*}
F_x&=
    \begin{bmatrix}
    x \\ x'
    \end{bmatrix}
    \sim 
    \mathcal{GP}\left(
    \begin{bmatrix}
    \mu_{x} \\ \mu_{x'}
    \end{bmatrix},
        \begin{bmatrix}
    k_x & k_x^\partial \\ ^\partial k_x & ^\partial k_x^\partial
    \end{bmatrix}
    \right),
    \\
    F_y&=
    \begin{bmatrix}
    y \\ y'
    \end{bmatrix}
    \sim 
    \mathcal{GP}\left(
    \begin{bmatrix}
    \mu_{y} \\ \mu_{y'}
    \end{bmatrix},
    \begin{bmatrix}
    k_y & k_y^\partial \\ ^\partial k_y & ^\partial k_y^\partial
    \end{bmatrix}
    \right),
\end{align*}
where $F_x$ and $F_y$ are the Gaussian process estimates for the positions and velocities of the agent in the $x$ and $y$ coordinates respectively. Their corresponding mean vectors are $M_x$ and $M_y$, while the covariance matrices are denoted by $K_x$ and $K_y$.

In the following theorem, we show that the future location of the obstacle can  be estimated, and the uncertainty is quantified in terms of the standard deviation, as a function of mean and variance, both of which are in polynomial basis.

\subsection{Trajectory Estimation}
Consider an ordered set of $n$ measurement times $T=\{t_1,\dots,t_n\}\in \mathcal{T}_O$.
At each time, a noisy measurement of position and velocity is generated for each dimension:
\begin{align*}
    x\in\mathbb{R}^{n}, \quad x'\in\mathbb{R}^{n}, \quad y\in\mathbb{R}^{n} \quad 
    \textnormal{and} \quad y'\in\mathbb{R}^{n}.
\end{align*}
The measurement vectors are augmented with the mean of the intention at the time $t_I$.
The augmented time vector then becomes $\bar{T}=\left[t_1,\dots,t_n,t_I\right]^\top$, and the augmented measurement vectors become
\begin{align*}
    \bar{x}&=[x^\top, \mu_{Ix}]^\top\in\mathbb{R}^{n+1}, 
    \\
    \bar{x'}&=[x'^\top, \mu_{Ix'}]^\top\in\mathbb{R}^{n+1}, 
    \\
    \bar{y}&=[x^\top, \mu_{Iy}]^\top\in\mathbb{R}^{n+1}, 
    \\
    \textnormal{and}  \quad
    \bar{y'}&=[y'^\top, \mu_{Iy'}]^\top\in\mathbb{R}^{n+1}.
\end{align*}
Next, this information is composed into two datasets of the form given in Equation~\eqref{eq:gp_prob}:
\begin{align*}
    D_x &= \{\bar{T},\,[\bar{x}^\top,\bar{x}'^\top]^\top\}, \quad \textnormal{and}
    \quad    D_y &= \{\bar{T},\,[\bar{y}^\top,\bar{y}'^\top]^\top\}.
\end{align*}
The posterior mean and covariance are then calculated using Equation~\eqref{eq:gp_prob}.
Numerical methods for calculating the matrix inverse can be found in~\cite{rasmussen06}.

\subsection{Estimation Basis}
\begin{lemma}
\label{th:1}
Consider the multi-output covariance function given in Equation~\eqref{eq:covF}:
\begin{align*}
    K =     
    \begin{bmatrix}
    k_f & k_f^\partial \\ ^\partial k_f & ^\partial k_f^\partial
    \end{bmatrix}.
\end{align*}
Given a data-set $D$, the elements of $k_f(\bar{T},t)$ and  $^\partial k _f(\bar{T},t)$, for any $t\in\mathcal{T}_{P}$, are polynomials of order less than or equal to three. 
\end{lemma}
\begin{proof}
Without loss of generality, let $\tau\equiv0$, then
\begin{align*}
    k_f(\bar{T},t) &= \theta^2_f \left[\frac{1}{3} \textnormal{ min}^3(\bar{T},t) +\frac{1}{2}\left|\bar{T}-t\right|\textnormal{ min}^2(\bar{T},t)\right].
\end{align*}
To evaluate $\mathrm{min}(\tilde{T},t)$, recall that $T=[t_1,\dots,t_n]\subset\mathcal{T}_O=[t_a,t_b]$ and that $t\in\mathcal{T}_P=(t_b,t_I]$.
Then we can simplify the expression element-wise for the vector $\bar{T}$:
\begin{align*}
    &k_f(\bar{T},t) = 
    \\
    &\theta^2_f \left(\frac{1}{3} 
    \begin{bmatrix}
    t_1^3 \\ t_2^3 \\ \vdots \\ t_{n}^3 \\ t^3  
    \end{bmatrix} 
    +\frac{1}{2}
    \begin{bmatrix}
    t-t_1 \\ t-t_2 \\ \vdots \\ t-t_{n} \\ t_I-t  
    \end{bmatrix}
    \begin{bmatrix}
    t_1^2 \\ t_2^2 \\ \vdots \\ t_{n}^2 \\ t^2  
    \end{bmatrix}
    \right)
    .
\end{align*}
Note that vector multiplication here is element-wise.
Similarly,
\begin{align*}
    ^\partial k_f (\bar{T},t) &= \theta_f\theta_{f'}\left[\mathbb{I}(t<\bar{T})t^2/2 + \mathbb{I}(t\geq \bar{T})(\bar{T}-\bar{T}^2/2)\right],
\end{align*}
 and we can simplify this equality under the same assumptions to be 
\begin{align*}
    & ^\partial k_f(\bar{T},t) = 
    \\
    &\theta_{f}\theta_{f'} \left(\frac{1}{2} 
    \begin{bmatrix}
    0 \\ 0 \\ \vdots \\ 0 \\ t^2
    \end{bmatrix} 
    +
    \begin{bmatrix}
     t t_{1}-t_{1}^2/2 \\ t t_{2}-t_{2}^2/2 \\ \vdots \\ t t_{n}-t_{n}^2/2 \\ 0  
    \end{bmatrix}
    \right).
\end{align*}
Clearly each element of these arrays is a polynomial of order less than or equal to three.
\end{proof}
\begin{lemma}[Mean Basis for Position]
\label{th:2}
Consider the posterior mean computed in Equation~\eqref{eq:gp_prob} given a data-set D:
\begin{align*}
    M(t) &= K(\bar{T},t) ^\top ( K(\bar{T},\bar{T}) + \Sigma^2)^{-1} Y,
\end{align*}
and the chosen covariance function as defined in Equation~\eqref{eq:covF}:
\begin{align*}
    K =     
    \begin{bmatrix}
    k_f & k_f^\partial \\ ^\partial k_f & ^\partial k_f^\partial
    \end{bmatrix}.
\end{align*}
Then the expected position trajectory is a cubic polynomial for any $t\in\mathcal{T}_{P}$.
\end{lemma}
\begin{proof}
    Let $P$ denote the precision matrix, partitioned as follows:
    \begin{align*}
        P = ( K(\bar{T},\bar{T}) + \Sigma^2)^{-1} = 
        \begin{bmatrix}
            P_{11} &  P_{12} \\ P_{12}^\top & P_{22}
        \end{bmatrix} .
    \end{align*}
    Note that this matrix has no dependence on $t$.
    Combining the precision matrix with the vector of measurements yields
    \begin{align*}
        \mu(t) = 
        \begin{bmatrix}
        \mu_f(t) \\ \mu_{f'}(t)
        \end{bmatrix} 
        = K(\bar{T},t)^\top
        \begin{bmatrix}
            P_{11}Y_f +  P_{12}Y_{f'} \\ P_{12}^\top Y_f + P_{22}Y_{f'}
        \end{bmatrix}.
    \end{align*}
    Considering only the position mean, the matrix multiplication becomes
    \begin{align*}
        \mu_f(t) = 
        \begin{bmatrix}
        k_f(\bar{T},t) & ^\partial k_f(\bar{T},t)
        \end{bmatrix}
        \begin{bmatrix}
            P_{11}Y_f +  P_{12}Y_{f'} \\ P_{12}^\top Y_f + P_{22}Y_{f'}
        \end{bmatrix}.
    \end{align*}
    Notice that the mean is a linear combination of $k_f(\bar{T},t)$ and $^\partial k_f(\bar{T},t)$. Since a linear combination of polynomials of order less than and equal to order three is a cubic polynomial, we have that $\mu_f$ is a cubic polynomial.
\end{proof}
\begin{lemma}[Variance Basis for Position]
\label{th:3}
Consider the posterior variance computed in Equation~\eqref{eq:gp_prob}, given a data-set D:
\begin{align*}
    \sigma^2(t) & = K(t,t)
    - K(\bar{T},t) ^\top P K(\bar{T},t),
\end{align*}
and the chosen covariance function, $K$, as defined in Equation~\eqref{eq:covF}.
Then for any $t\in\mathcal{T}_{P}$, the position variance, $\sigma^2_f(t)$, is a sixth-order polynomial.
\end{lemma}
\begin{proof}
    The variance function for position is given by the first block in the posterior covariance function:
    \begin{align*}
        \sigma^2_{f}(t)&=k_f(t,t) 
        \\&- k_f(\bar{T},t) \left(k_f(\bar{T},t) P_{11} +^\partial k_f(\bar{T},t) P_{12} \right)
        \\&+^\partial k_f(\bar{T},t) \left(k_f(\bar{T},t)  P_{12}^\top + ^\partial k_f(\bar{T},t) P_{22} \right).
    \end{align*}
    Clearly, $k_f(t,t)$ is cubic by substitution of $t$ into Equation~\eqref{eq:cubic}. By Lemma~\ref{th:1}, the elements of both $k_f(\bar{T},t)$ and $^\partial k_f(\bar{T},t)$ are polynomials. Since the posterior variance function is a product of these polynomial elements with a linear combination of these elements, the resulting function will be a sixth-order polynomial. 
\end{proof}
\begin{theorem}[Uncertainty Region Boundary Basis]
\label{th:4}
The boundary of the obstacle uncertainty region over the prediction time interval, $\mathcal{T}_P$, is given by the sum of a third-order polynomial and the square root of a sixth-order polynomial:
\begin{align}
        C_x(t) = \mu_x(t) \pm 2\sqrt{\sigma^2_x(t)}, \quad
        C_y(t) = \mu_y(t) \pm 2\sqrt{\sigma^2_y(t)}.   
        \label{eq:confidence}
\end{align}
\end{theorem}
\begin{proof}
    By Lemma~\ref{th:2} we have that the posterior mean $\mu_f$ over the interval $\mathcal{T}_P$ is a cubic polynomial. By Lemma~\ref{th:3} we have that the variance, $\sigma^2_f$, is a sixth-order polynomial. We construct the confidence interval in each dimension by adding the square root of the variance and mean functions as shown in Equation~\eqref{eq:confidence}. Therefore we have a confidence interval described by a known basis.
\end{proof}

\begin{figure*}[t!]
    \centering
    \subfloat[$t=1.25$s]{\label{fig:merging_seq1}\includegraphics[width=0.5\columnwidth]{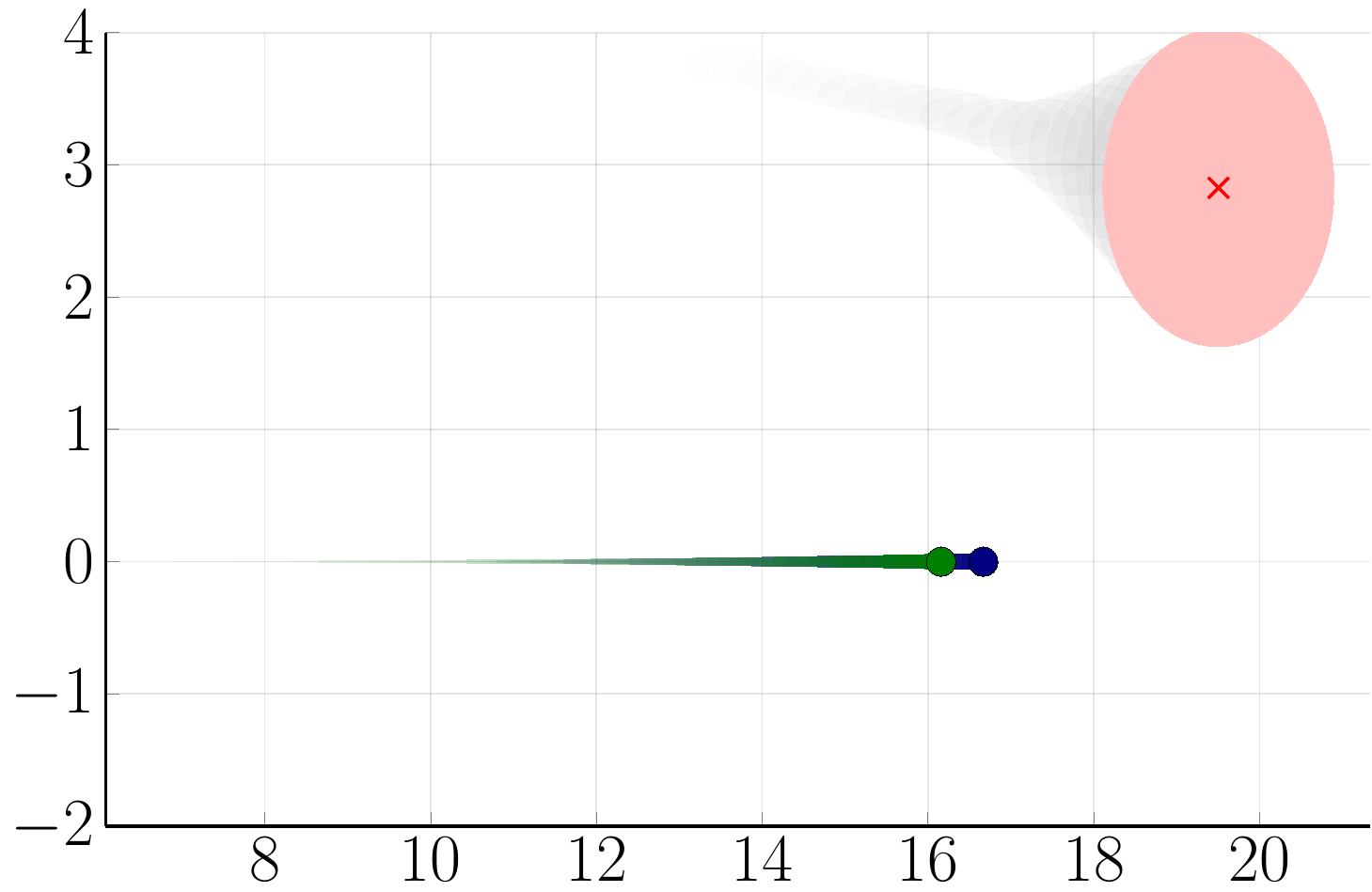}}
    \subfloat[$t=1.75$s]{\label{fig:merging_seq2}\includegraphics[width=0.5\columnwidth]{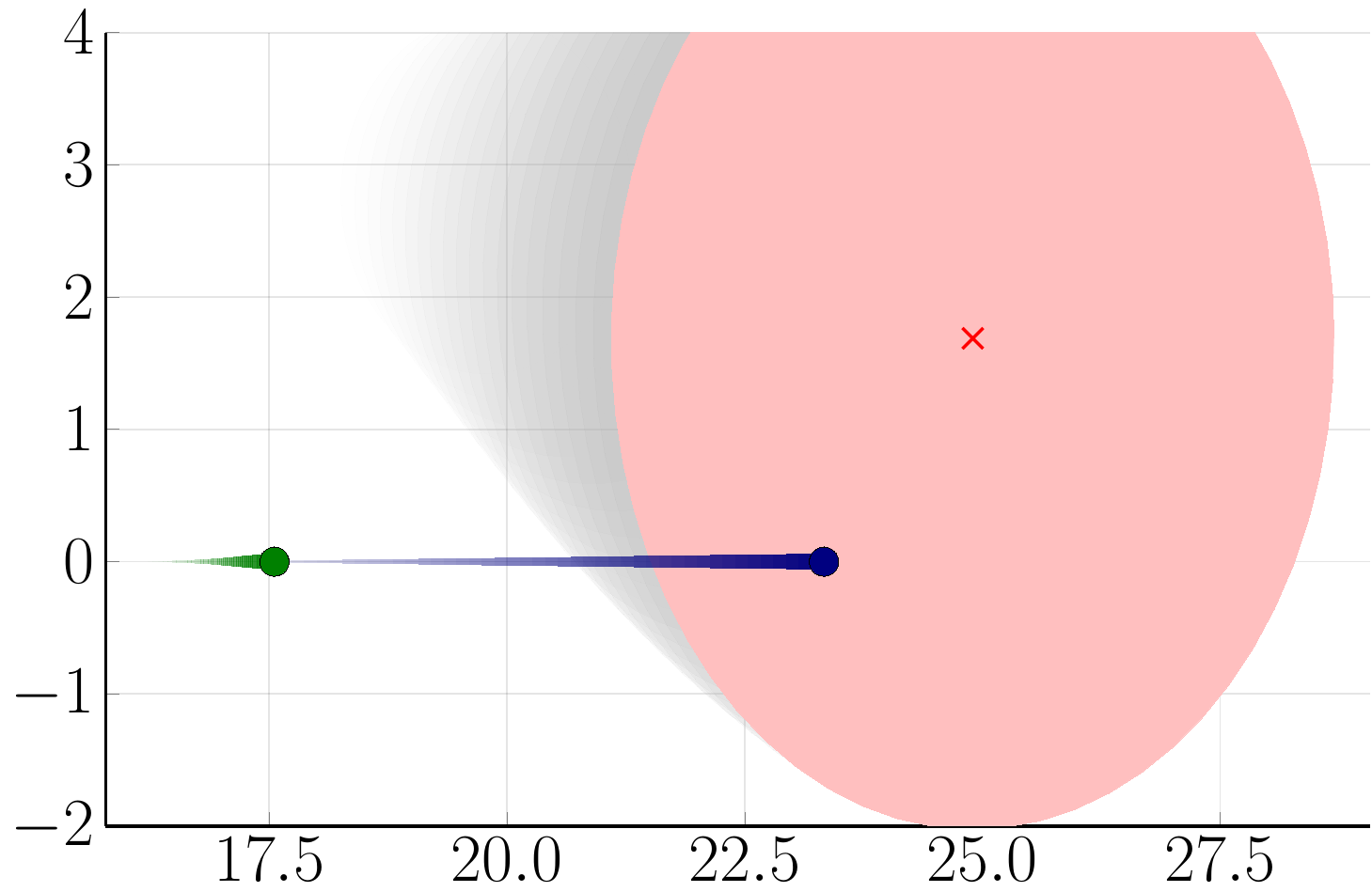}} 
    \subfloat[$t=2.25$s]{\label{fig:merging_seq3}\includegraphics[width=0.5\columnwidth]{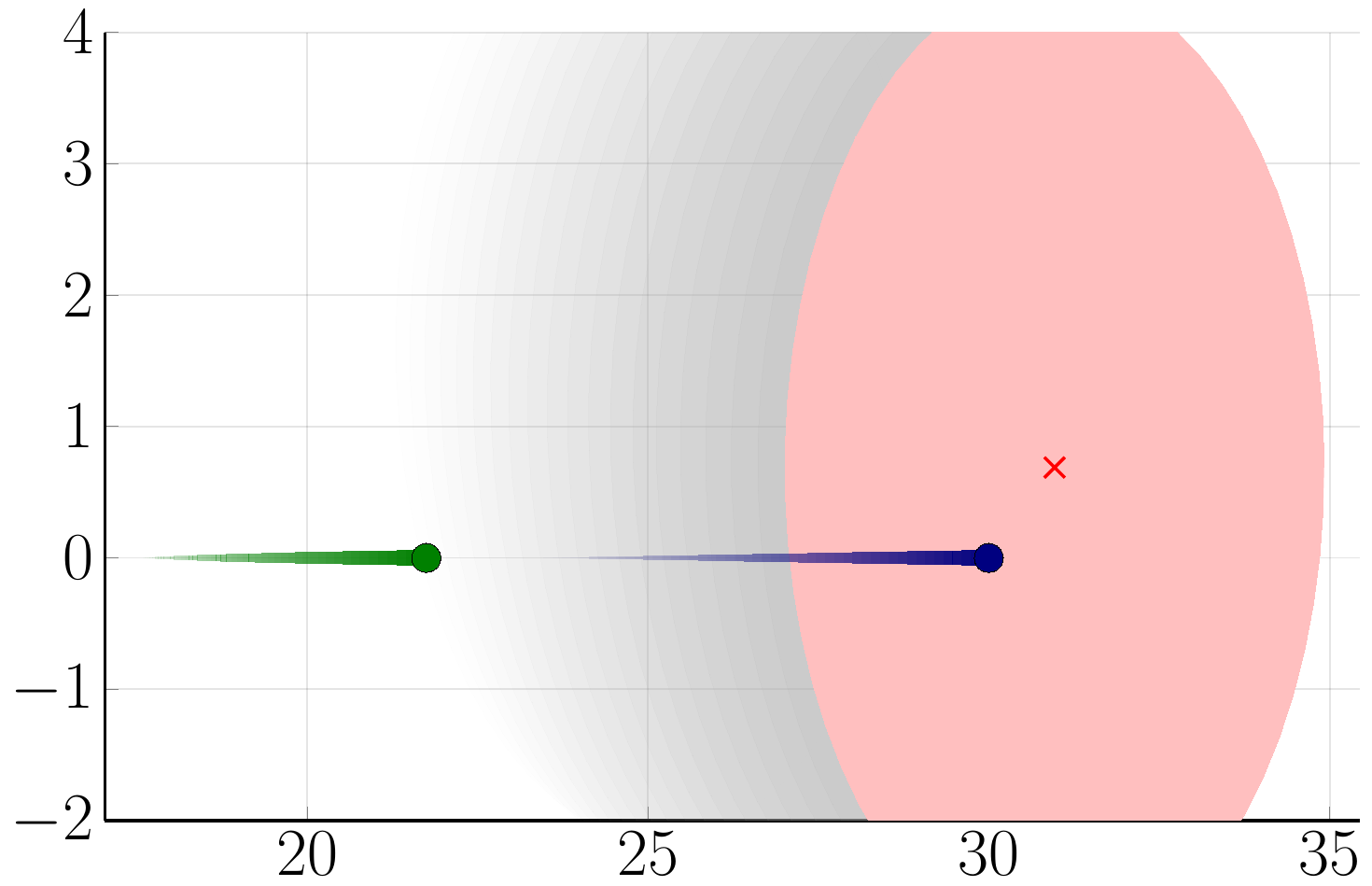}} 
    \subfloat[$t=2.75$s]{\label{fig:merging_seq4}\includegraphics[width=0.5\columnwidth]{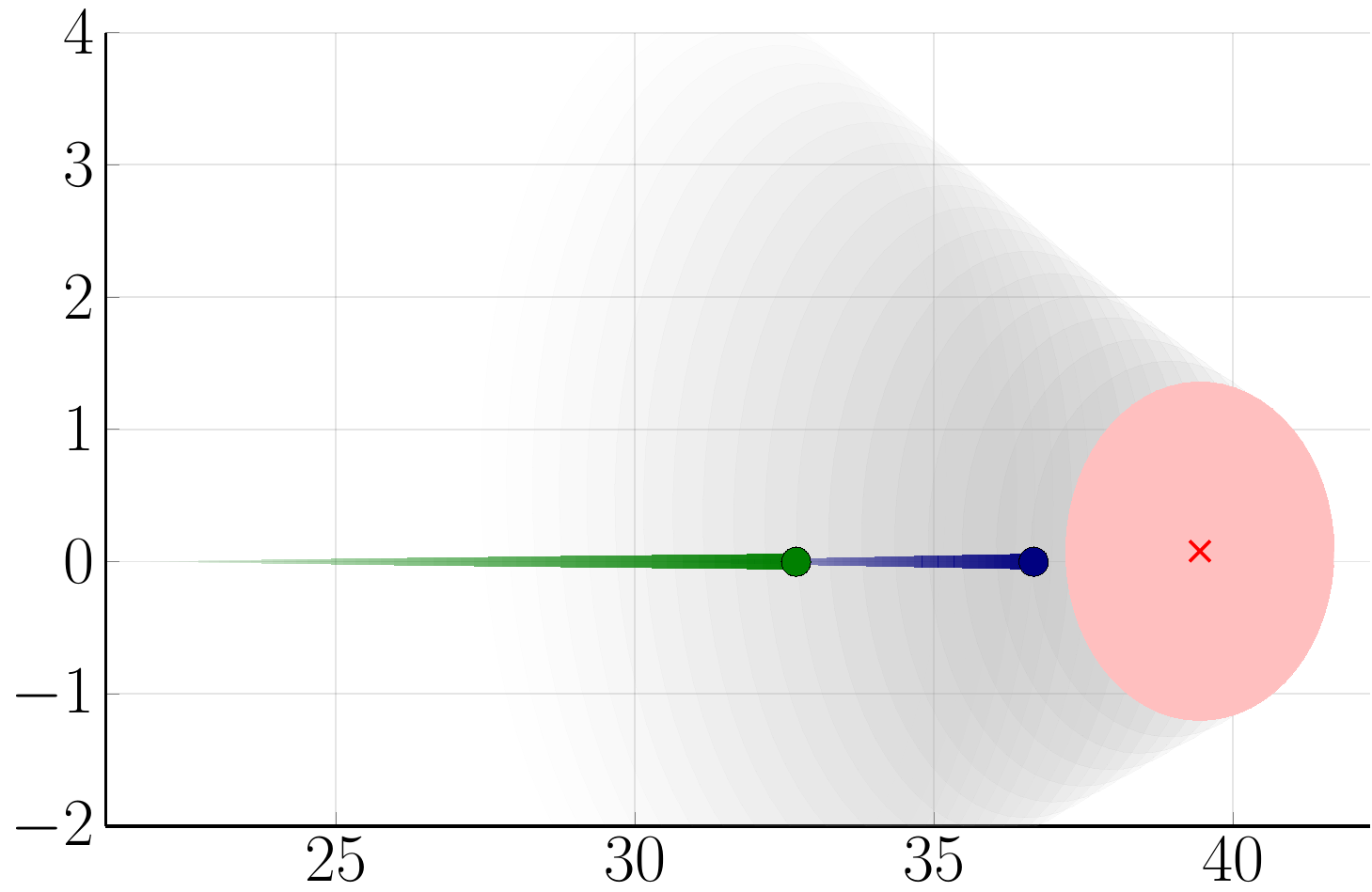}} 
    \caption{Simulation of first scenario shown at different times. The expected obstacle location is shown as a red ``x" with the uncertainty region of the obstacle shown as a red ellipse. The blue and green points show two possible agent locations at that time with tails showing their past trajectory.}
    \label{fig:merging1}
\end{figure*}

\begin{figure}[htbp]
    \centering
    \subfloat[Obstacle and agent trajectories in $x$-$y$ plane.]{\label{fig:merging_xy}\includegraphics[width=\columnwidth]{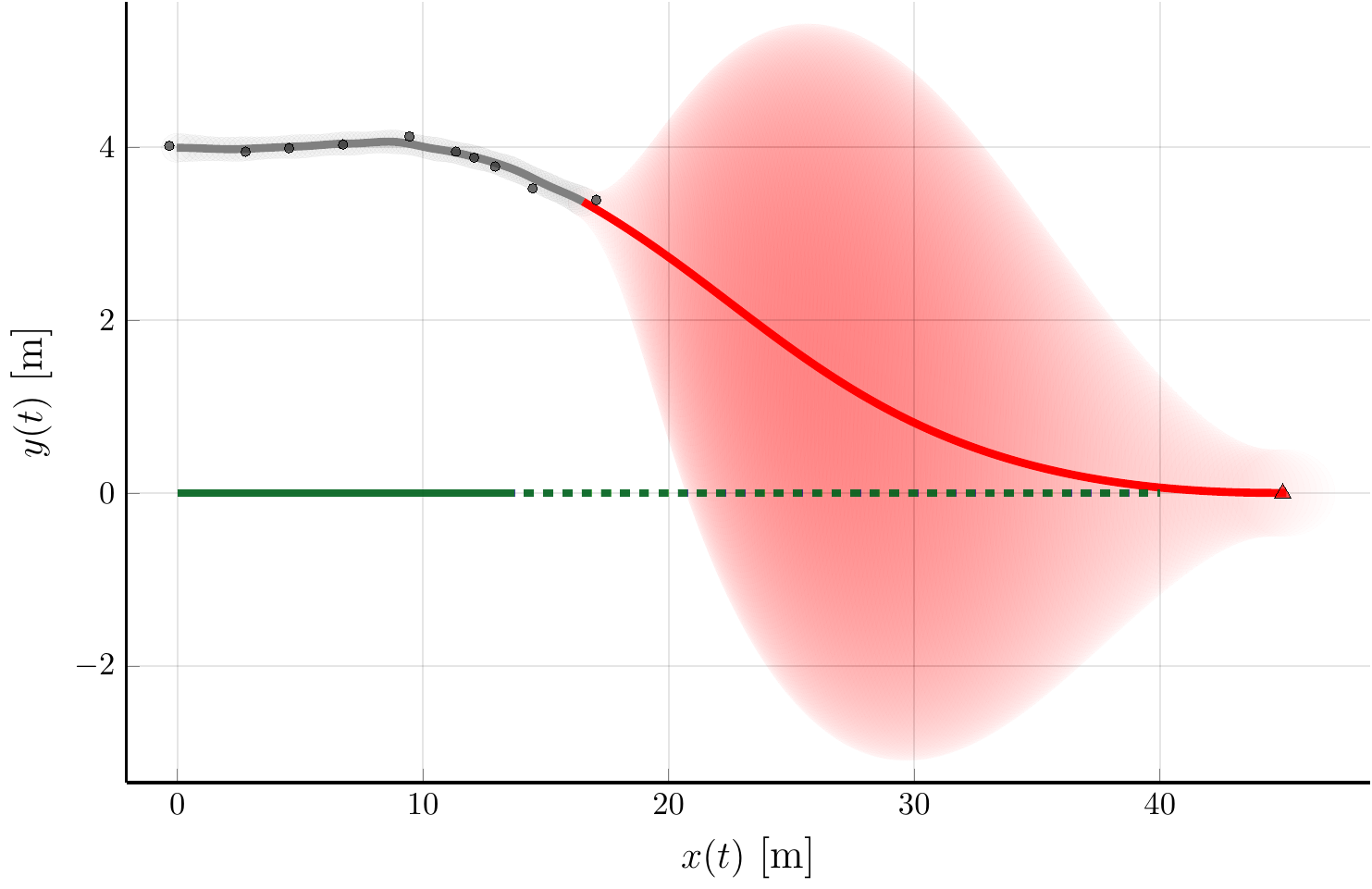}} \\
    \subfloat[Time history of trajectories\newline in $x$ dimension.]{\label{fig:merging_x}\includegraphics[width=0.5\columnwidth]{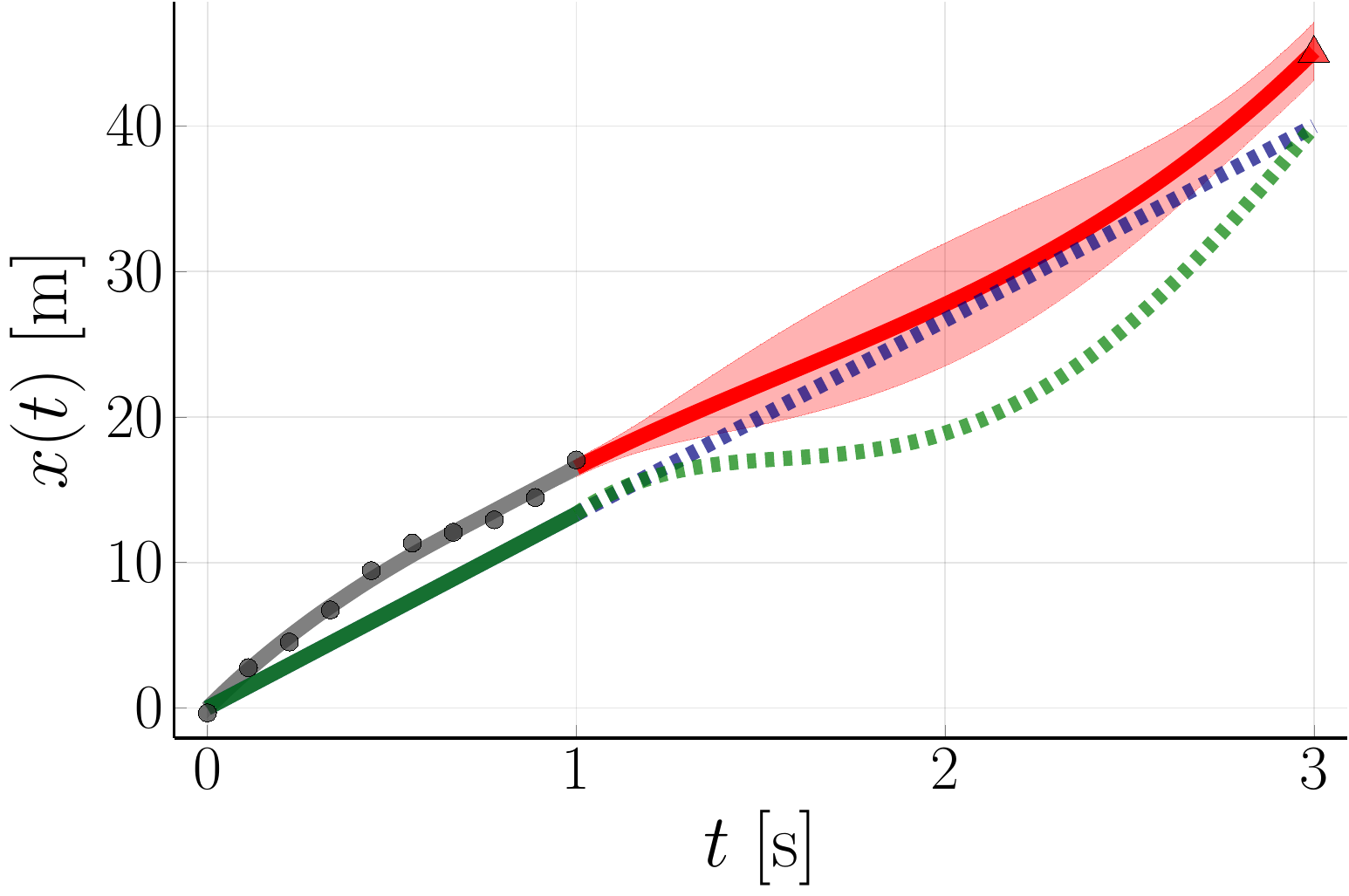}} 
    \subfloat[Time history of trajectories\newline in $y$ dimension.]{\label{fig:merging_y}\includegraphics[width=0.5\columnwidth]{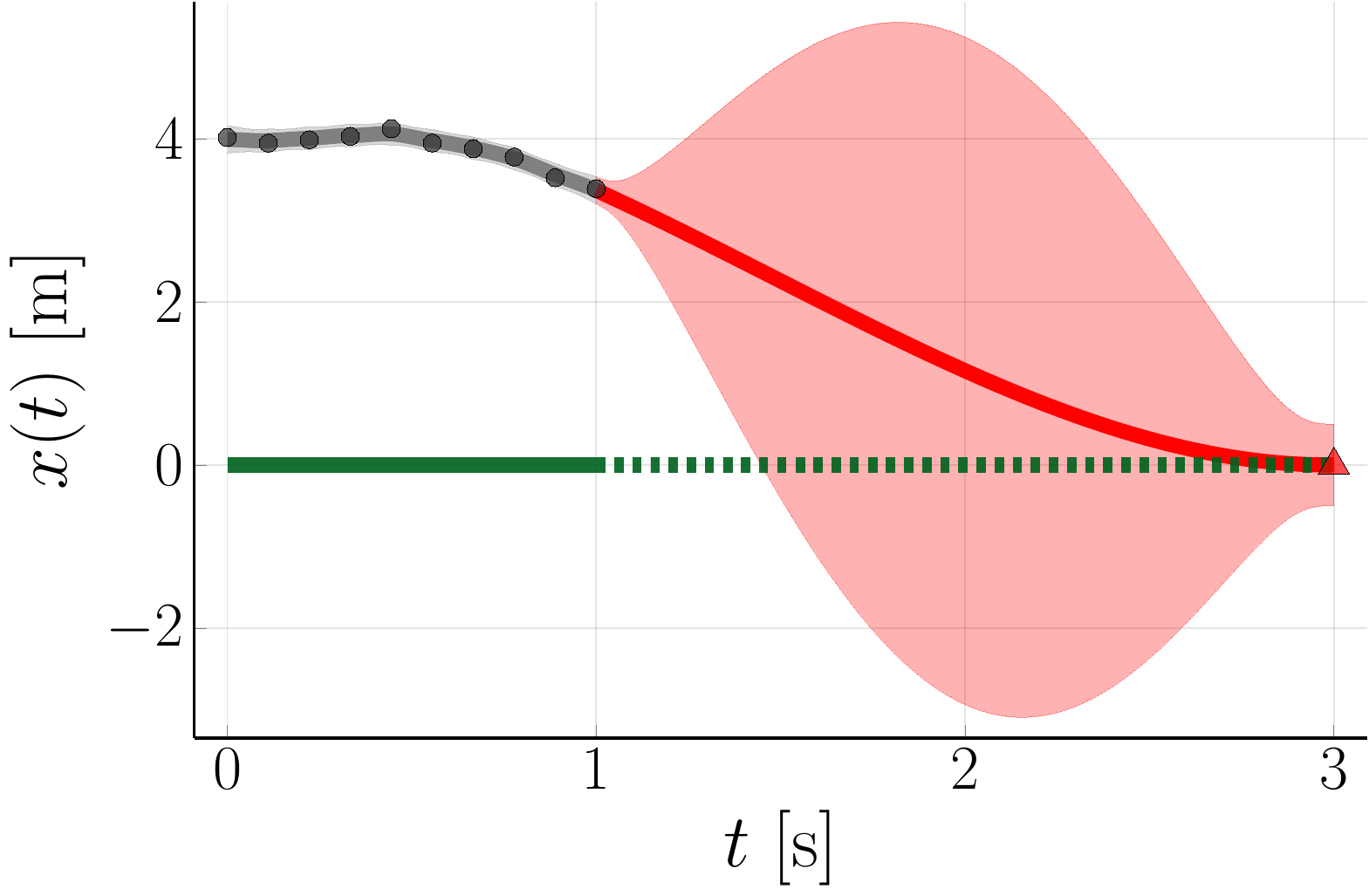}} \\
    \subfloat[Minimum distance for each proposed trajectory to uncertainty region over prediction time horizon.]{\label{fig:merging_d}\includegraphics[width=\columnwidth]{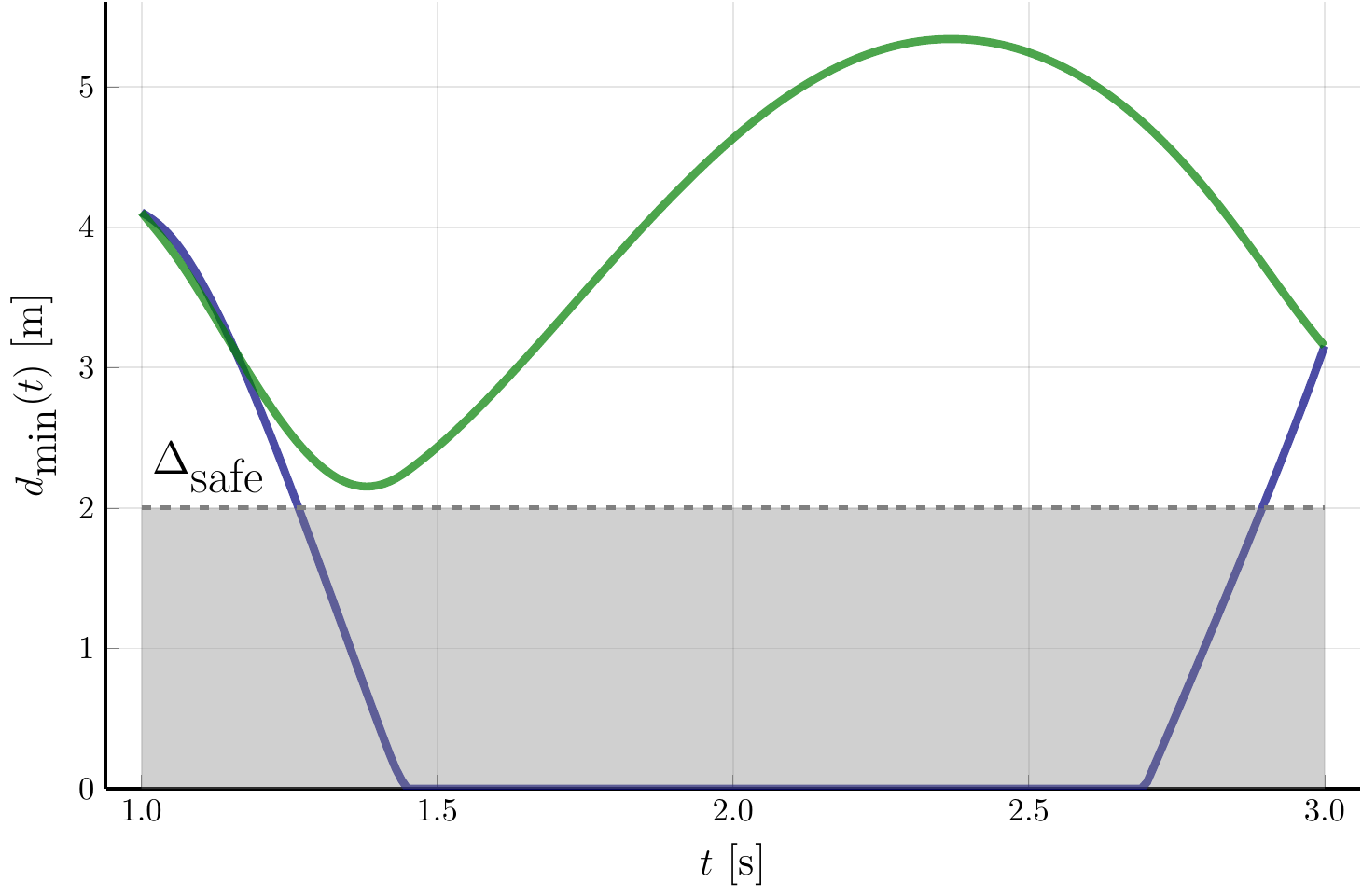}} 
\caption{Simulation of the first scenario.}
    \label{fig:merging2}
\end{figure}

\subsection{Parameterization}
To predict the future position of the vehicle, the matrix $K(T,t)$ would need to be sampled or converted into a parametric form. Since a third-order polynomial can be fit with four samples and a sixth-order polynomial with seven samples, we sample the mean and covariance functions at uniformly spaced times $\mathbf{T}_\mu\subset \mathcal{T}_P$ and $\mathbf{T}_\sigma\subset \mathcal{T}_P$.
These polynomials are exactly equivalent to the mean and variance functions, for $t\in\mathcal{T}_P$, but are renamed to indicate their form:
\begin{align}
\label{eq:param}
    &\psi_{\mu x}(t)= \mu_x(t),
    \quad \psi_{\mu y}(t)= \mu_y(t), \nonumber
    \\
    &\psi_{\sigma^2 x}(t)= \sigma^2_x(t) \quad \textnormal{and}
    \quad\psi_{\sigma^2 y}(t)= \sigma^2_y(t).
\end{align}

\subsection{Collision Prediction}
We employ the collision prediction methods described in \cite{arun19} to find the intersection between the boundary of the uncertainty region, given in \cref{eq:confidence}, and the agent's parametric trajectory, \cref{eq:param}. This method is computationally efficient and relies on interval optimization techniques to find the point intersection without sacrificing accuracy.  Note that the nature of global optimization methods only allows for detecting one intersection even if the parametric curves intersect several times. However, since trajectory estimation is only performed for short time horizons, we make a reasonable assumption that the agent's trajectory only intersects the uncertainty region once. Using this framework, we obtain the interval $\mathcal{T}_x^C \subset \mathcal{T}_P$, which indicates the time interval of collision in the $x$ dimension, so that 
\begin{align*}
    d_\textrm{min}(t, \psi_x, \Psi_x^{2\sigma}) < \Delta_\textrm{safe}
\end{align*}
holds true for all $t \in \mathcal{T}_x^C$, and $\mathcal{T}_y^C \subset \mathcal{T}_P$ indicates the time interval of collision in the $y$ dimension, so that
\begin{align*}
    d_\textrm{min}(t, \psi_y, \Psi_y^{2\sigma}) < \Delta_\textrm{safe} 
\end{align*}
holds  for all $t \in \mathcal{T}_y^C$. When $\mathcal{T}_x^C \cap \mathcal{T}_y^C = \{\varnothing\}$, then it is easy to see that either $d_\textrm{min}(t, \psi_x, \Psi_x^{2\sigma}) \ge \Delta_\textrm{safe}$ or $d_\textrm{min}(t, \psi_y, \Psi_y^{2\sigma}) \ge \Delta_\textrm{safe}$ for all $t \in \mathcal{T}_P$, which implies that \cref{eq:prob} is not violated.

\begin{figure*}[t]
    \centering
    \subfloat[$t=1.25$s]{\label{fig:crossing_seq1}\includegraphics[width=0.5\columnwidth]{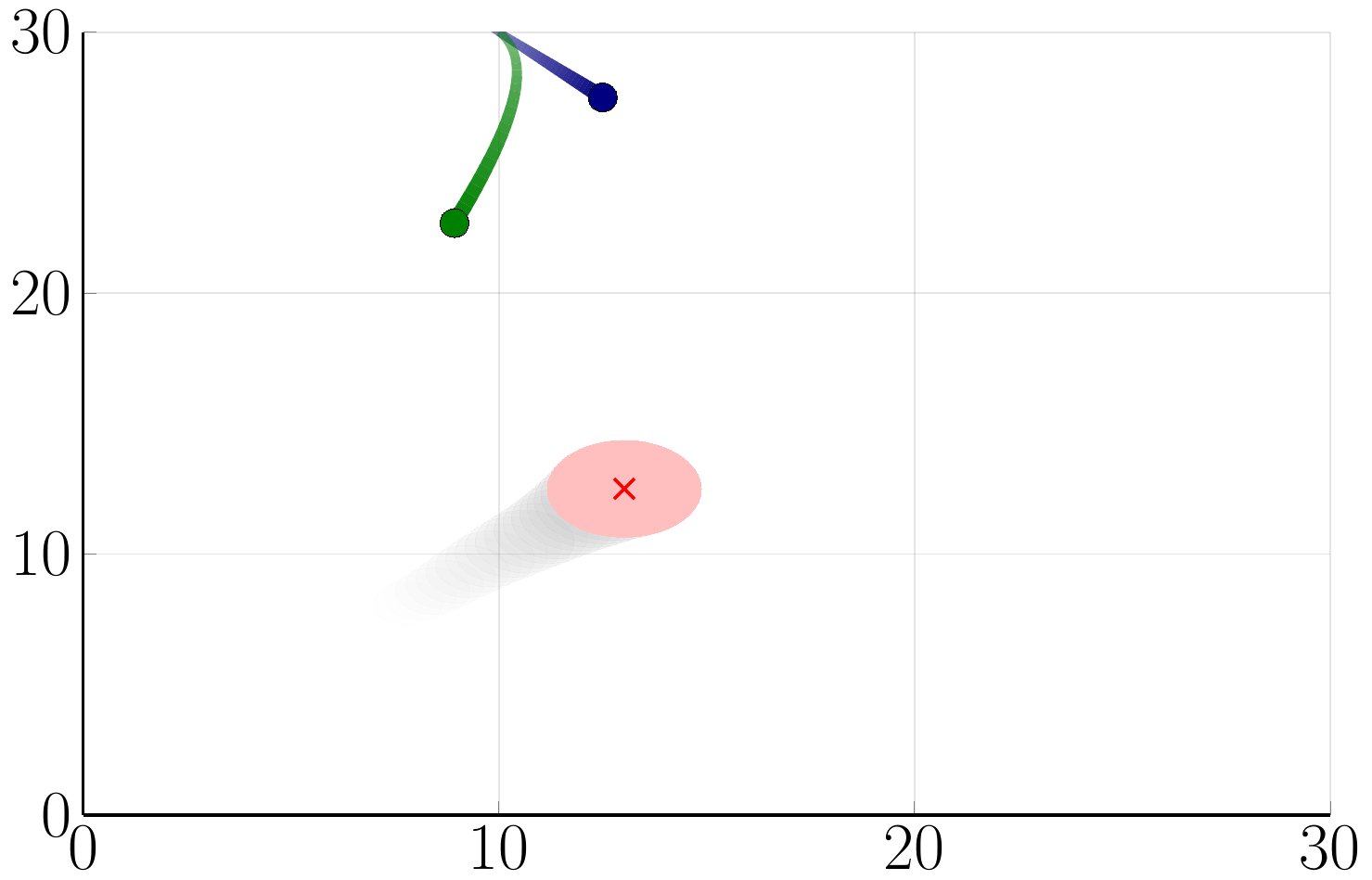}}
    \subfloat[$t=1.75$s]{\label{fig:crossing_seq2}\includegraphics[width=0.5\columnwidth]{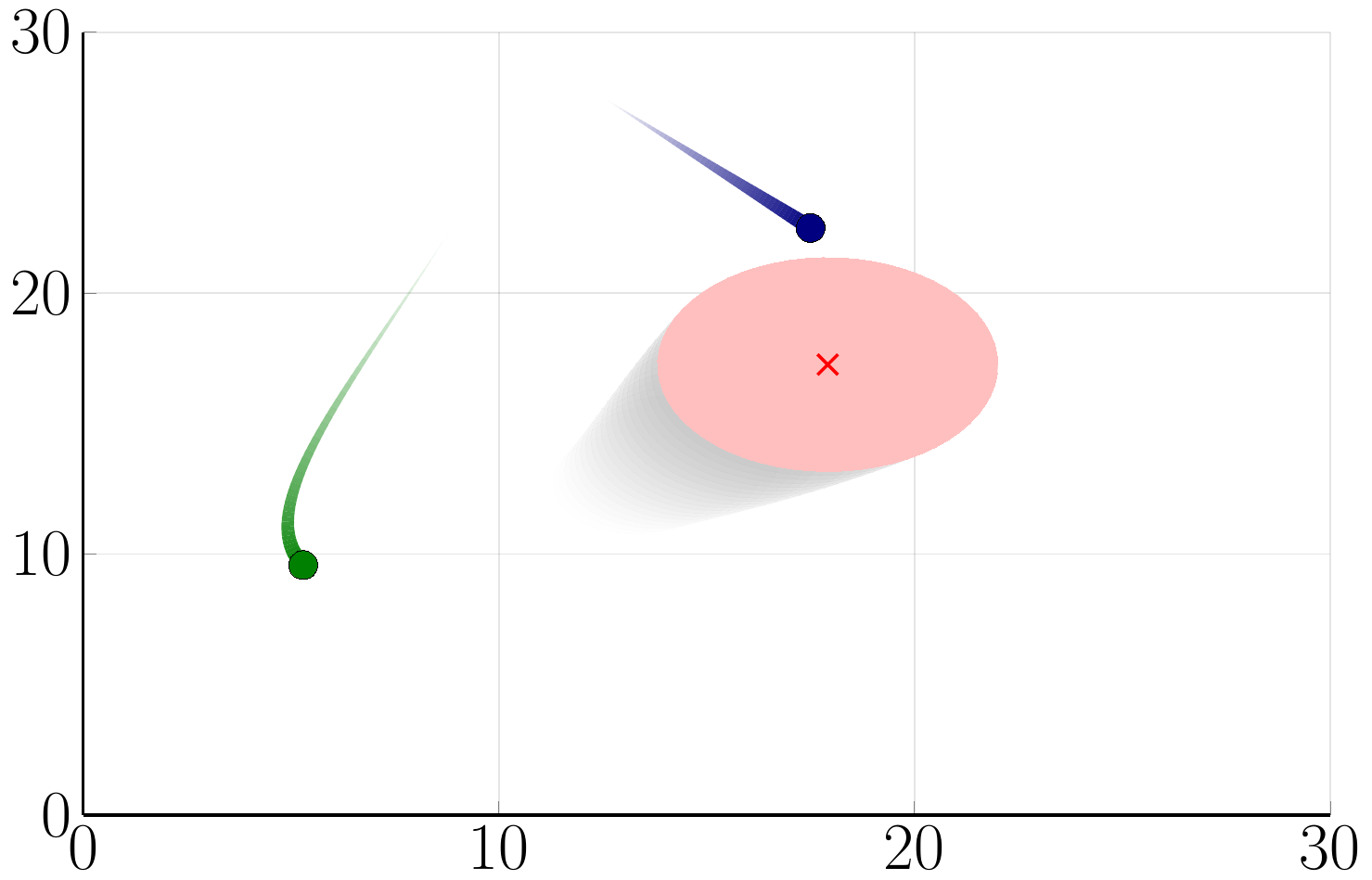}} 
    \subfloat[$t=2.25$s]{\label{fig:crossing_seq3}\includegraphics[width=0.5\columnwidth]{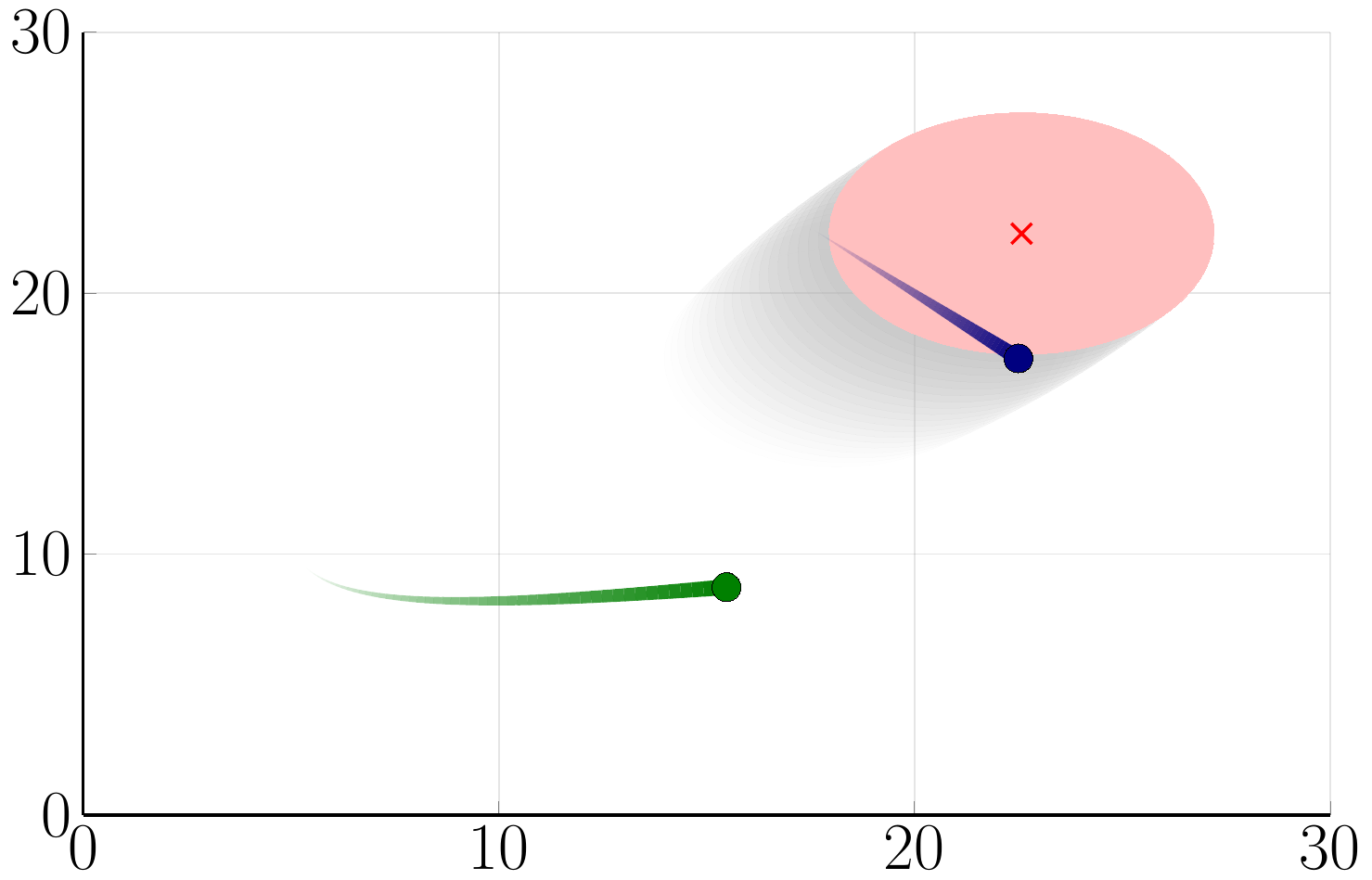}} 
    \subfloat[$t=2.75$s]{\label{fig:crossing_seq4}\includegraphics[width=0.5\columnwidth]{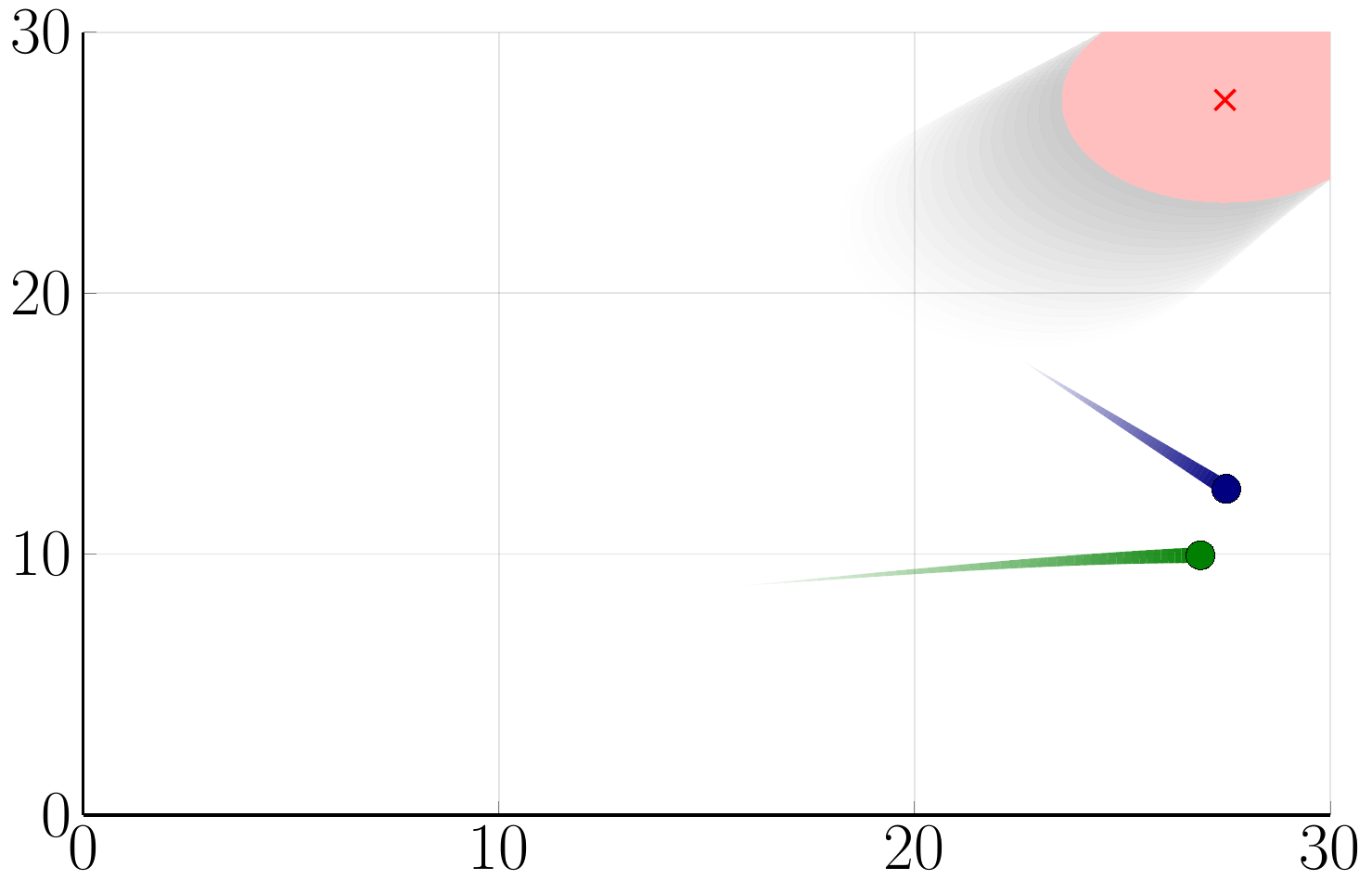}} 
    \caption{Simulation of second scenario shown at different times. The expected obstacle location is shown as a red ``x" with the uncertainty region of the obstacle shown as a red ellipse. The blue and green points show two possible agent locations at that time with tails showing their past trajectory.}
    \label{fig:crossing1}
\end{figure*}

\begin{figure}[htbp]
    \centering
    \subfloat[Obstacle and agent trajectories in $x$-$y$ plane.]{\label{fig:crossing_xy}\includegraphics[width=\columnwidth]{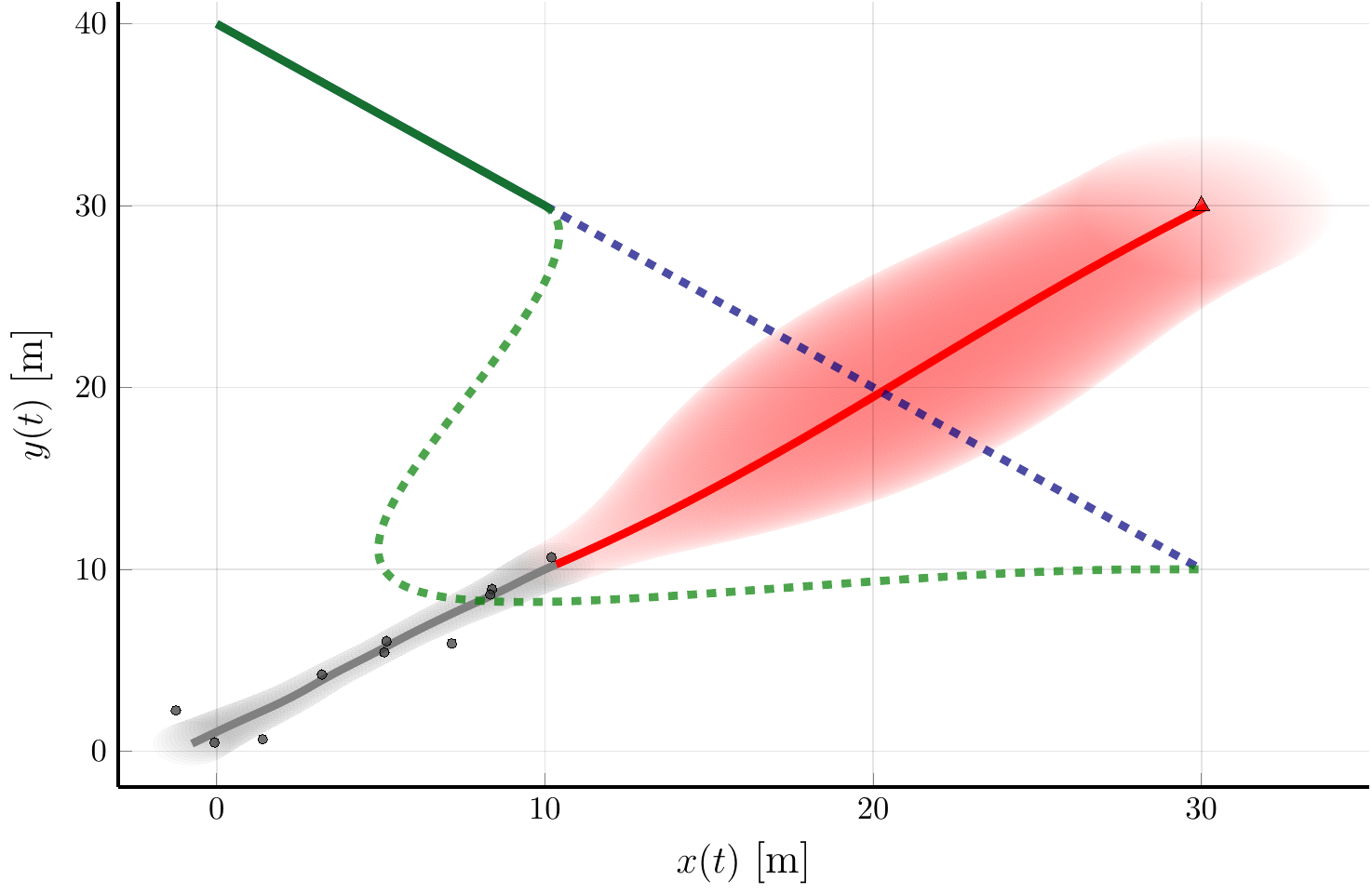}} \\
    \subfloat[Time history of trajectories\newline in $x$ dimension.]{\label{fig:crossing_x}\includegraphics[width=0.5\columnwidth]{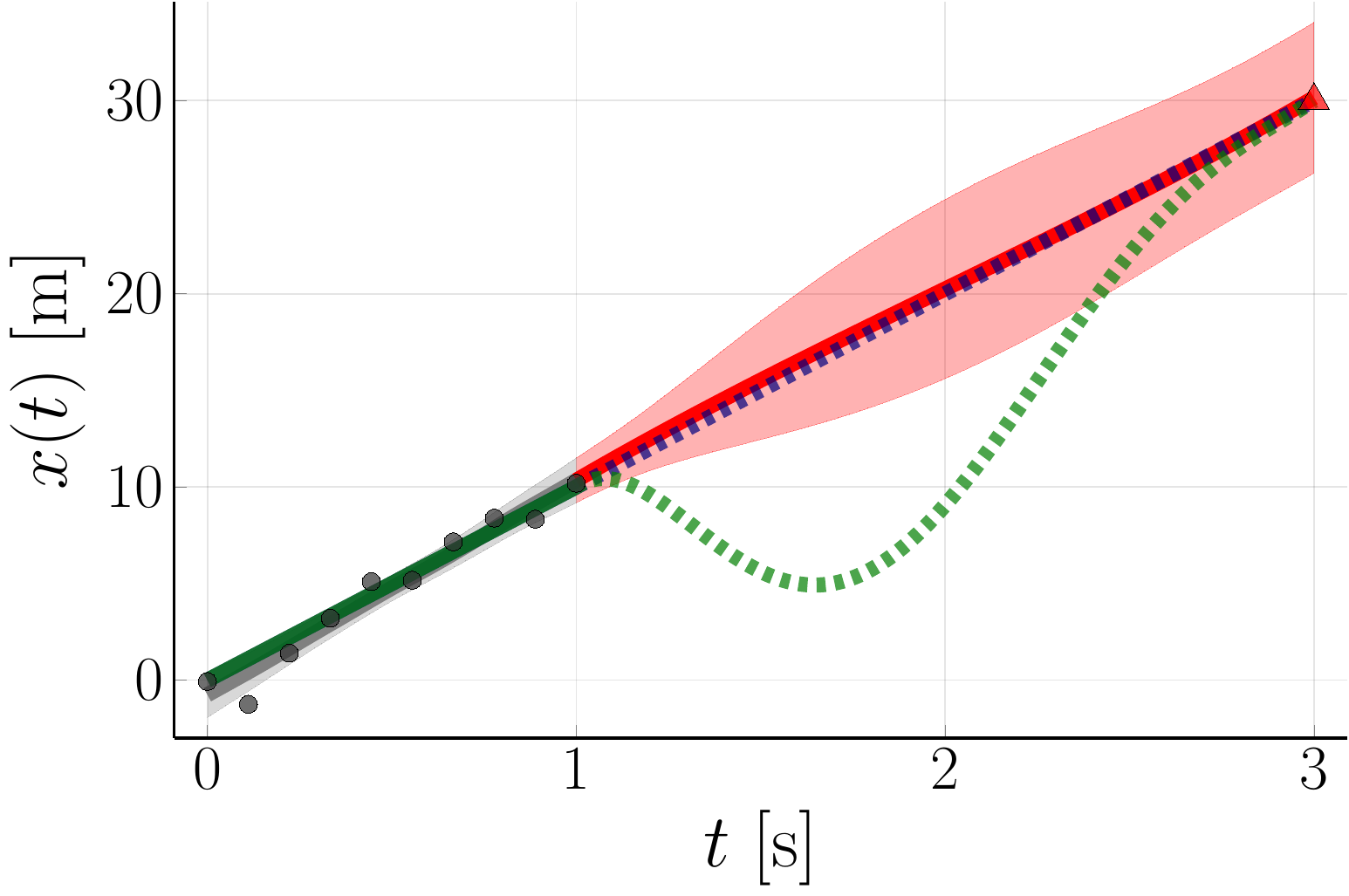}}
    \subfloat[Time history of trajectories\newline in $y$ dimension.]{\label{fig:crossing_y}\includegraphics[width=0.5\columnwidth]{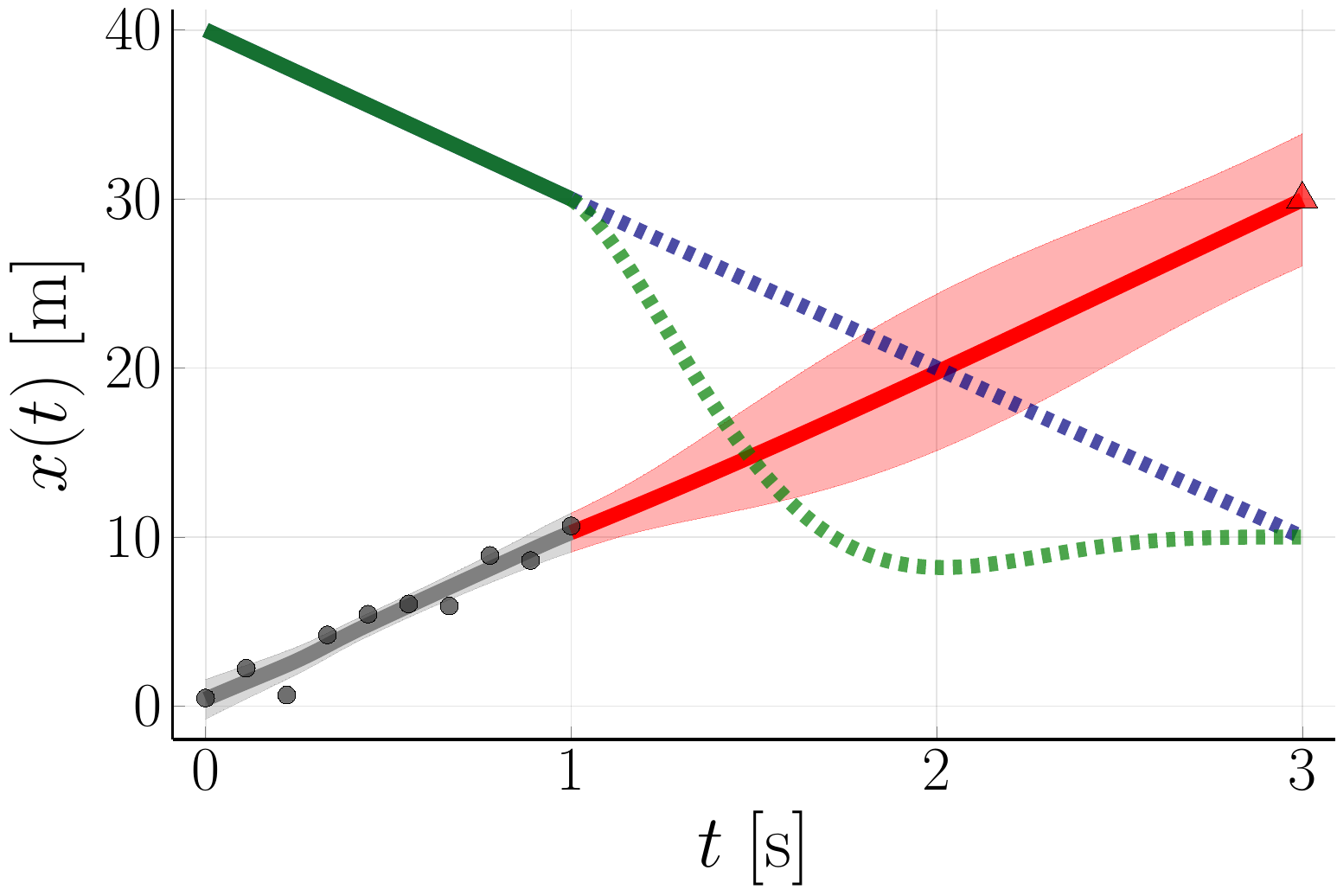}} \\
    \subfloat[Minimum distance for each proposed trajectory to uncertainty region over prediction time horizon.]{\label{fig:crossing_d}\includegraphics[width=\columnwidth]{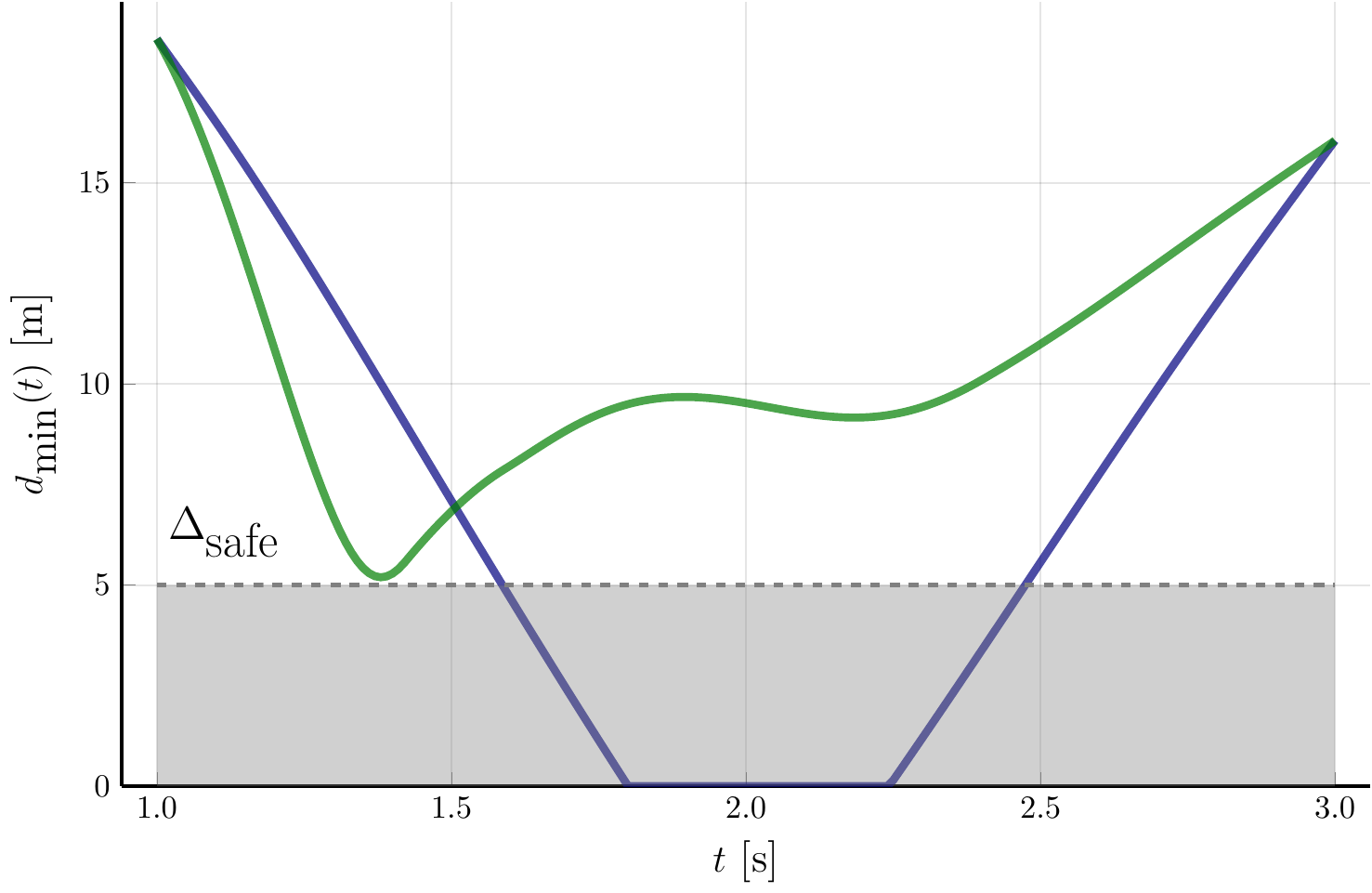}}
    \caption{Simulation of the second scenario.}
    \label{fig:crossing2}
\end{figure}

\section{RESULTS}
\label{s:results}
Two simulations are considered to demonstrate the application of the proposed Gaussian process prediction method. The first is a merging scenario, where an obstacle moves from a parallel lane into the agent's lane in front of the agent. The second scenario is a perpendicular cross, where the obstacle crosses in front of the agent. In each case, the agent considers two planned trajectories and checks to see which maintains the safety distance away from the predicted obstacle uncertainty region.

In both situations, the observed data is shown as gray points, and the intention as a red triangle. Based on this information, the confidence region is constructed and shown in red, with the solid red line being the expected future obstacle trajectory. The blue and green lines show the agent's past trajectory, with dotted lines showing possible future trajectories.

The following parameters are fixed in these simulations. The time interval of interest is fixed, with $t_a=0$, $t_b=1$ and $t_I=3$. We assume that we have $n=10$ samples on the observation time interval. The kernel parameters for both $x$ and $y$ dimensions are $\theta_f = 10$, $\theta_{f'} = 30$ and $\tau = 11$.

\subsection{Merging with Speed Adjustment}
In the first scenario, shown in Figure~\ref{fig:merging_xy}, the obstacle begins to move in the $x$-$y$ plane towards the agent. The assumed intention is that the vehicle will overtake and merge in front of the agent at an appropriate distance ahead, shown as the red triangle. The agent considers two possible trajectories, one that reduces its $x$ velocity and one that does not. In Figure~\ref{fig:merging_x}, we see that the green trajectory, which alters the vehicle's speed, is outside the red, shaded obstacle uncertainty region for the full prediction horizon. The blue trajectory shows that a constant $x$ velocity would lead to the agent entering the uncertainty region. The safety distance  is calculated over time in Figure~\ref{fig:merging_d}. The green trajectory, with altered $x$ velocity remains greater than $\Delta_{\textrm{safe}}$, the blue trajectory violates the safety distance constraint, leading to a predicted collision.

In Figure~\ref{fig:merging1}, the uncertainty region of the obstacle is shown at four times during the prediction interval in the $x$-$y$ plane. The red cross indicates the expected obstacle location, and the red ellipse is the uncertainty region. The light gray shadow of the uncertainty region shows the past $0.5s$ of uncertainty regions. This interval is colored gray to indicate that it is past data and not actively avoided. The vehicle locations are given as blue and green points, with trails showing the past $0.5s$ interval of the trajectory. Clearly the green points do not intersect with the uncertainty region. The blue points, representing the vehicle location if it were to hold a constant velocity, are inside the uncertainty region in Figures~\ref{fig:merging_seq2} and~\ref{fig:merging_seq3}.

In this simulation the safety distance is $2$~m. The measurement noise is zero mean with a variance of $0.25$ in both position and velocity in the $x$~axis and a variance of $0.01$ in both position and velocity in the $y$~axis. The intention uncertainty is set with a variance of $1$ for both position and velocity in the $x$~axis, and $0.0625$ for both position and velocity in the $y$~axis.
\subsection{Perpendicular Cross with Path Re-planning}
In the second scenario, shown in Figure~\ref{fig:crossing_xy}, the obstacle moves in the $x$-$y$ plane towards the agent. The assumed intention is that the vehicle will continue its current path, shown as the red triangle. The agent considers two possible trajectories, one that responds to the obstacle and one that does not. In Figures~\ref{fig:crossing_x} and~\ref{fig:crossing_y}, we see that the green trajectory, that avoids the obstacle, is adjusted so that the $x$ and $y$ components of the trajectory are not simultaneously in the uncertainty region. The blue trajectory shows that a constant $x$ velocity would lead to the agent entering the uncertainty region. Since this blue trajectory is inside the $x$ uncertainty region for all of $\mathcal{T}_P$, there is no action that could be taken in the $y$ dimension that would avoid predicted collision. The safety distance  is calculated over time in Figure~\ref{fig:crossing_d}. The green trajectory, with altered $x$ and $y$ trajectories, remains greater than $\Delta_{\textrm{safe}}$, while the blue trajectory violates the safety distance constraint, leading to a predicted collision.

In Figure~\ref{fig:crossing1}, the uncertainty region of the obstacle is shown at four times during the prediction interval in the $x$-$y$ plane. The red cross indicates the expected obstacle location, and the red ellipse is the uncertainty region. The light gray shadow of the uncertainty region shows the past $0.5s$ of uncertainty regions. This interval is colored gray to indicate that it is past data and not actively avoided. The vehicle locations are given as blue and green points, with trails showing the past $0.5s$ interval of the trajectory. Clearly the green points do not intersect with the uncertainty region. The blue point, representing the vehicle location if it were to continue its previous behavior, is inside the uncertainty region in Figure~\ref{fig:merging_seq3}.

In this simulation, the safety distance is $5$~m and the measurement noise is zero mean with a variance of $1$ in position and  $4$ in velocity. The intention uncertainty is set with a variance of $4$ for position and $16$ for velocity.

\section{CONCLUSION}
\label{s:con}
In this paper, we have presented a Gaussian process based method to predict collision with uncertainty quantification. We show that with this approach the uncertainty regions for the future location of the obstacle can be modeled and  parameterized, which allows collisions to be predicted efficiently. With the assumption of a known yet probabilistic intention, the conservatism of the method is reduced. The presented simulations show that the method can be used to construct an estimate for an unknown obstacle's future position based solely on measurements of past position and velocity. The data-based approach reduces the knowledge of the obstacle required for many trajectory estimation methods.

The current method is formulated as a prediction at a single time instance rather than a series of times. We wish to extend this work by considering predictions over a series of times with two components: (1) a method of updating hyperparameters and (2) a method of handling dynamic intentions. The prediction at each step relies on the choice of  the scaling hyperparameters for both position and velocity in two dimensions.  In future work, we wish to compare standard hyperparameter optimization to hyperparameter choice based on the physical limitations of the obstacle. Additionally, the assumption of an a priori known intention at a fixed time is restrictive in a dynamic environment. Future work will investigate propagating the intention as the output of an intent estimation procedure.

\section{ACKNOWLEDGMENTS}
This work is supported by the National Aeronautics and Space Administration (NASA), National Science Foundation (NSF) National Robotics Initiative (NRI) award \#1820639 and NSF NRI award \#1528036.
\bibliographystyle{ieeetr}
\bibliography{main}

\end{document}